\documentclass[11pt]{article}

\usepackage[margin=0.8in]{geometry}
\usepackage{amsmath,amssymb,amsthm,mathtools}
\usepackage[T1]{fontenc}
\usepackage[hidelinks]{hyperref}
\usepackage{microtype}
\usepackage{enumitem}

\allowdisplaybreaks
\setlist{nosep,leftmargin=*}

\newtheorem{theorem}{Theorem}[section]
\newtheorem{proposition}[theorem]{Proposition}
\newtheorem{lemma}[theorem]{Lemma}
\newtheorem{corollary}[theorem]{Corollary}
\theoremstyle{definition}

\theoremstyle{remark}
\newtheorem{remark}[theorem]{Remark}

\newcommand{\Hh}{\mathcal H}
\newcommand{\E}{\mathbb E}
\newcommand{\Pp}{\mathbb P}
\newcommand{\one}{\mathbf 1}
\newcommand{\ip}[2]{\left\langle #1,#2\right\rangle}
\newcommand{\norm}[1]{\left\lVert #1\right\rVert}
\newcommand{\argminop}{\operatorname*{arg\,min}}
\newcommand{\conv}{\operatorname{conv}}
\newcommand{\supp}{\operatorname{supp}}
\newcommand{\Risk}{R_P}

\title{
Sharp margin-based generalization bounds for realizable SVM
}
\author{Steve Hanneke and Aryeh Kontorovich}
\date{\today}

\begin{document}
\maketitle

\begin{abstract}
Let the exact homogeneous hard-margin support vector machine be trained on
\(m\) independent observations from a Borel probability law on a real Hilbert
space. We prove that, with score zero counted as an error, there is a universal
numerical constant \(C\) such that
\[
 \Pp\left(
   \gamma_m>0,\quad
   \Risk(u_m)>
   \frac{C}{m}
   \left(
      K_m+\log\frac1\delta
   \right)
 \right)
 \le \delta .
\]
Here \(\gamma_m\) is the empirical homogeneous margin, \(u_m\) is the exact
minimum-norm unit-margin separator, \(r_m\) is the largest training radius,
and \(K_m:=r_m^2\norm{u_m}^2=r_m^2/\gamma_m^2\) on
\(\{\gamma_m>0\}\).

The proof is driven by a deterministic deletion problem. Given vectors
\(x_1,\ldots,x_n\) in the unit ball, delete a set \(B\) of constraints and
let \(u_B\) be the closest point to the origin that satisfies every retained
unit-margin constraint. Suppose that \(\norm{u_B}^2\le k\) and that every
deleted vector has nonpositive score under \(u_B\). We prove that a family of
such deletion sets of cardinality \(q\) has size at most
\(\exp(8k+2q)\).

The conceptual step is an exact identity obtained from the KKT
representation of \(u_B\). For a random deletion set, the identity converts
the mean squared spread of the separators into a weighted sum of score
deficits. It therefore forces a coordinate whose deletion status separates
the two conditional means by a quantitatively large amount. Revealing that
coordinate decreases the conditional separator variance enough to control
the binary entropy of the split. An entropy induction gives the deletion
count, and an exact factorial ghost-sample identity converts that count into
the stated high-probability SVM bound.
\end{abstract}

\section{Introduction and geometric overview}
\label{sec:introduction}

The result concerns the exact homogeneous hard-margin SVM; ``homogeneous''
means that there is no intercept term. The final statistical argument is
short once one understands a deterministic question about deleting
constraints. This section gives the geometric picture and the proof
architecture. Section~\ref{sec:definitions} then defines every object used in
the paper before the theorem statements.

\subsection{Why deletion sets appear}

Imagine drawing \(m+\ell\) labeled observations. Choose \(\ell\) of them as
a prospective test set and train the SVM on the remaining \(m\). If every
chosen test point lies within the training radius and is misclassified by
the separator trained on the complement, then the chosen indices form a
special deletion set. An exact exchangeability identity shows that the
\(\ell\)-th moment of the in-radius test error is the expected number of such
deletion sets, divided by \(\binom{m+\ell}{\ell}\).

The statistical problem is therefore reduced to the following deterministic
one: how many subsets can simultaneously behave as valid sets of deleted,
misclassified constraints while the separator trained on the complement has
bounded squared norm?

\subsection{The geometry of one deletion state}

Fix vectors \(x_1,\ldots,x_n\) in the unit ball. Each retained vector defines
a closed halfspace in parameter space,
\[
 H_i:=\{u\in\Hh:\ip{u}{x_i}\ge1\}.
\]
For a deletion set \(B\), the retained feasible region is
\[
 \mathcal C_B=\bigcap_{i\notin B}H_i.
\]
Whenever this intersection is nonempty, \(u_B\) is its closest point to the
origin. Thus deleting constraints enlarges the feasible region, and \(u_B\)
records how its nearest point moves.

The additional condition
\[
 \ip{u_B}{x_i}\le0\qquad(i\in B)
\]
is much stronger than merely omitting the deleted constraints. A retained
index has score at least \(1\), while a deleted index has score at most
\(0\). Hence two distinct admissible deletion sets necessarily produce
different separators. Geometrically, each changed deletion bit forces a
unit-sized score discrepancy in the direction of the corresponding vector.

\subsection{The conceptual turning point}

The closest feasible point has an exact KKT representation
\[
 u_B=\sum_{i\notin B}\alpha_i(B)x_i,
 \qquad
 \alpha_i(B)\ge0,
 \qquad
 \sum_i\alpha_i(B)=\norm{u_B}^2.
\]
The last equality is crucial. The bound \(\norm{u_B}^2\le k\) is therefore
not merely a radius bound on the cloud of possible separators: it is also a
bound on the total positive coefficient mass available in every state.

Now randomize the deletion set and write
\[
 U:=u_B,\qquad
 \bar u:=\E U,\qquad
 t:=\E\norm{U-\bar u}^2.
\]
Averaging the KKT representation gives the exact identity
\[
 t=
 \sum_i \E\alpha_i(B)
 \left(1-\ip{\bar u}{x_i}\right).
\]
This is the main insight of the proof. It says that the total squared spread
of the separator cloud must be witnessed by coordinates at which the mean
separator falls below unit score.

For a variable coordinate \(i\), let \(p_i\) be its deletion probability and
let \(g_i\) be the difference between its mean score when retained and its
mean score when deleted. Retention gives score at least \(1\), deletion gives
score at most \(0\), so \(g_i\ge1\). The identity above and the coefficient
mass bound imply that, whenever \(t>0\), some coordinate satisfies
\[
 p_i g_i\ge\frac{t}{k}.
\]
Conditioning on whether that coordinate is deleted separates the two
conditional mean separators by at least \(g_i\) in the direction \(x_i\).
The law of total variance then shows that revealing this one bit removes at
least
\[
 p_i(1-p_i)g_i^2
\]
from the unexplained separator variance.

This is the ``aha'' step: a combinatorial bit of deletion information is
forced to create a geometric drop in variance, and the exact KKT identity
finds a coordinate for which that drop is large enough.

\subsection{How the proof closes}

The proof has four transparent stages.

\begin{enumerate}
\item A concave function
\[
 \Phi_k(t)=t\log\frac{e^2k}{t}
\]
is calibrated so that the decrease of \(4\Phi_k(t)\) under the informative
split controls the binary entropy of that split. If a coordinate is deleted
with probability greater than \(1/2\), its binary entropy is instead bounded
directly by twice its deletion probability.

\item Induction over the variable coordinates gives
\[
 H(B)\le4\Phi_k(t)+2\E|B|
       \le8k+2\E|B|.
\]

\item For the uniform law on a fixed-cardinality family
\(\mathcal F\subseteq\binom{[n]}q\), this becomes
\[
 \log|\mathcal F|\le8k+2q.
\]

\item The factorial ghost-sample identity turns this deterministic count into
a factorial-moment bound for the in-radius population error. A separate
order-statistic estimate controls the population mass beyond the largest
training radius.
\end{enumerate}

Throughout the paper, \(k\) is called the \emph{squared-norm budget}: it is
an upper bound on \(\norm{u_B}^2\).

No support-vector cardinality estimate and no sample-compression theorem is
used. In particular, the proof is independent of the support-vector
assertion in Theorem~19 of Hanneke and Kontorovich~\cite{HK21}.

\section{Notation, definitions, and theorem statements}
\label{sec:definitions}

\subsection{General conventions}

The symbol \(\Hh\) denotes the ambient real Hilbert space,
\(\ip{\cdot}{\cdot}\) its inner product, and \(\norm{\cdot}\) its norm. Its
closed unit ball is
\[
 \mathbb B_{\Hh}:=\{x\in\Hh:\norm x\le1\}.
\]
For a positive integer \(n\),
\[
 [n]:=\{1,\ldots,n\},
 \qquad
 \binom{[n]}q:=\{B\subseteq[n]:|B|=q\}.
\]
For a finite set \(V\), \(2^V\) denotes its power set. The indicator of an
event \(E\) is \(\one_E\). The symbols \(\Pp\) and \(\E\) denote probability
and expectation under the law currently being discussed. The logarithm is
natural, so all Shannon entropies are measured in nats.

If \(\nu\) is a probability law on \(2^V\), then
\[
 \supp(\nu):=\{A\subseteq V:\nu(A)>0\},
 \qquad
 H_\nu(B):=-\sum_{A\subseteq V}\nu(A)\log\nu(A).
\]
Zero-mass terms are omitted. When the law is clear, we write simply \(H(B)\).

\subsection{The statistical SVM model}

Let \(P\) be a Borel probability law on
\(\Hh\times\{-1,+1\}\). Its input marginal is assumed to be carried by a
separable closed linear subspace \(\Hh_0\subseteq\Hh\). For a sample size
\(m\ge1\), draw
\[
 S_m=((X_1,Y_1),\ldots,(X_m,Y_m))\sim P^m,
\]
where \(P^m\) is the \(m\)-fold product law. Define the signed sample vectors
and the largest training radius by
\[
 Z_i:=Y_iX_i,
 \qquad
 r_m:=\max_{1\le i\le m}\norm{X_i}.
\]

The empirical homogeneous margin is
\[
 \gamma_m
 :=
 \sup_{\norm w=1}\min_{1\le i\le m}\ip{w}{Z_i}.
\]
On the event \(\{\gamma_m>0\}\), the unit-margin constraints are feasible,
and the exact hard-margin SVM parameter is
\[
 u_m
 :=
 \argminop_{u\in\Hh}
 \left\{
   \frac12\norm u^2:
   \ip{u}{Z_i}\ge1\quad(1\le i\le m)
 \right\}.
\]
We set \(u_m:=0\) on \(\{\gamma_m\le0\}\). The random normalized squared norm
is
\[
 K_m:=r_m^2\norm{u_m}^2.
\]
Thus \(K_m\) is a statistic of the training sample; it is not the
deterministic budget \(k\) used in the deletion theorem below.

For a realized sample, the population misclassification risk, with score
zero counted as an error, is
\[
 \Risk(u_m)
 :=
 \int_{\Hh\times\{-1,+1\}}
 \one_{\{y\ip{u_m}{x}\le0\}}\,P(dx,dy).
\]
Equivalently, it is the conditional error probability on an independent test
pair \((X,Y)\sim P\).

We split this population risk at the random radius \(r_m\):
\[
 R_m^{\le}
 :=
 \int
 \one_{\{y\ip{u_m}{x}\le0,\ \norm x\le r_m\}}\,P(dx,dy),
\]
\[
 \rho_m
 :=
 \int
 \one_{\{\norm x>r_m\}}\,P(dx,dy)
 =
 P(\norm X>r_m).
\]
Then
\[
 \Risk(u_m)\le R_m^{\le}+\rho_m.
 \tag{2.1}\label{eq:risk-split}
\]
The quantity \(\rho_m\) is a \emph{population} tail probability evaluated at
the empirical maximum \(r_m\). It is not a tail probability under the
empirical measure; under the empirical measure the mass above \(r_m\) would
be zero.

The confidence parameter in the terminal theorem is
\(\delta\in(0,1)\).

The separable-support assumption is used only for measurability. On the
full-probability event that all sample points belong to \(\Hh_0\), orthogonal
projection onto \(\Hh_0\) preserves every constraint value. Hence
feasibility, the minimum-norm separator, the margin, \(K_m\), and the risk
agree with the corresponding quantities computed in \(\Hh_0\). The Borel
versions used here are justified in Appendix~\ref{sec:convex-appendix}. The
assumption is automatic when \(\Hh\) is separable.

\subsection{The deterministic deletion model}

Fix \(n\ge1\) and deterministic vectors
\[
 x_1,\ldots,x_n\in\mathbb B_{\Hh}.
\]
A set \(B\subseteq[n]\) is called a \emph{deletion set}; its complement
\[
 A_B:=[n]\setminus B
\]
is the retained index set. The feasible region after deleting \(B\) is
\[
 \mathcal C_B
 :=
 \left\{
   u\in\Hh:
   \ip{u}{x_i}\ge1\quad(i\in A_B)
 \right\}.
\]
If \(\mathcal C_B\ne\varnothing\), its unique minimum-norm element is
\[
 u_B
 :=
 \argminop_{u\in\mathcal C_B}\frac12\norm u^2.
\]
When \(A_B=\varnothing\), the convention is
\(\mathcal C_B=\Hh\) and \(u_B=0\).

Let \(k>0\). A deletion state \(B\) is \emph{admissible at squared-norm
budget \(k\)} if
\[
 \mathcal C_B\ne\varnothing,
 \qquad
 \norm{u_B}^2\le k,
 \qquad
 \ip{u_B}{x_i}\le0\quad(i\in B).
 \tag{2.2}
\]
Here \(k\) is a real upper bound on the squared separator norm. It is neither
a dimension nor a number of observations.

A fixed-cardinality deletion family is a set
\[
 \mathcal F\subseteq\binom{[n]}q,
\]
where the integer \(q\in\{0,\ldots,n\}\) is the number of deleted
coordinates in every state. Thus \(\mathcal F\) is a collection of subsets,
not a probability distribution. In the uniform deletion-family theorem,
every \(B\in\mathcal F\) is required to satisfy (2.2).

\subsection{Random deletion laws}

The entropy induction needs a slightly more general model that is stable
under conditioning. Let \(J\) and \(V\) be disjoint finite index sets.
Indices in \(J\) are permanently retained, while indices in \(V\) may be
retained or deleted. Given vectors \(x_i\in\mathbb B_{\Hh}\) for
\(i\in J\cup V\) and a set \(B\subseteq V\), define
\[
 \mathcal C_B^{J,V}
 :=
 \left\{
   u\in\Hh:
   \ip{u}{x_i}\ge1
   \quad(i\in J\cup(V\setminus B))
 \right\},
\]
and, whenever this set is nonempty,
\[
 u_B:=\argminop_{u\in\mathcal C_B^{J,V}}\frac12\norm u^2.
\]
The superscript is suppressed because \(J\) and \(V\) remain fixed within
each application.

A \emph{deletion law} (or deletion measure) is simply a probability law
\(\nu\) on \(2^V\). It is unrelated to the empirical distribution of the
training observations. The law \(\nu\) is \emph{admissible at squared-norm
budget \(k\)} if every \(B\in\supp(\nu)\) satisfies
\[
 \mathcal C_B^{J,V}\ne\varnothing,
 \qquad
 \norm{u_B}^2\le k,
 \qquad
 \ip{u_B}{x_i}\le0\quad(i\in B).
 \tag{2.3}
\]

For \(B\sim\nu\), define the random separator, its mean, and its mean squared
spread by
\[
 U:=u_B,
 \qquad
 \bar u:=\E U,
 \qquad
 t:=\E\norm{U-\bar u}^2.
 \tag{2.4}
\]
The quantity \(t\) is the trace of the covariance of the Hilbert-valued
random variable \(U\); concretely, it is the average squared distance of the
possible separators from their mean. Since
\[
 t=\E\norm U^2-\norm{\bar u}^2,
\]
admissibility implies \(0\le t\le k\).

For each \(B\in\supp(\nu)\), Appendix~\ref{sec:convex-appendix} supplies
nonnegative KKT coefficients \(\alpha_i(B)\) satisfying
\[
 U=\sum_{i\in J\cup V}\alpha_i(B)x_i,
 \qquad
 \alpha_i(B)=0\quad(i\in B),
 \tag{2.5}
\]
\[
 \alpha_i(B)>0
 \ \Longrightarrow\
 \ip{U}{x_i}=1,
 \qquad
 \sum_{i\in J\cup V}\alpha_i(B)=\norm U^2\le k.
 \tag{2.6}
\]
The certificate need not be unique; any one may be chosen separately for
each state.

For a variable coordinate \(i\in V\), write
\[
 p_i:=\Pp(i\in B).
\]
When \(0<p_i<1\), define its conditional retained score, conditional deleted
score, and score jump by
\[
 a_i:=\E[\ip{U}{x_i}\mid i\notin B],
 \qquad
 b_i:=\E[\ip{U}{x_i}\mid i\in B],
 \qquad
 g_i:=a_i-b_i.
 \tag{2.7}
\]
Admissibility gives \(a_i\ge1\), \(b_i\le0\), and hence \(g_i\ge1\).

Finally, define the binary entropy
\[
 h(p):=-p\log p-(1-p)\log(1-p),
 \qquad 0\le p\le1,
 \tag{2.8}
\]
with \(0\log0:=0\), and the concave variance potential
\[
 \Phi_k(t)
 :=
 \begin{cases}
 t\log\dfrac{e^2k}{t},&t>0,\\[2mm]
 0,&t=0,
 \end{cases}
 \qquad 0\le t\le k.
 \tag{2.9}
\]
Its derivatives on \((0,k]\) are
\[
 \Phi_k'(t)=\log\frac{ek}{t},
 \qquad
 \Phi_k''(t)=-\frac1t.
 \tag{2.10}
\]
Thus \(\Phi_k\) is increasing and concave on \([0,k]\), and
\[
 \Phi_k(t)\le\Phi_k(k)=2k.
 \tag{2.11}
\]

\subsection{Ghost-sample notation}

The statistical reduction uses the following notation. Fix integers
\(m,\ell\ge1\), set
\[
 N:=m+\ell,
\]
and draw the pooled sample
\[
 W=((X_1,Y_1),\ldots,(X_N,Y_N))\sim P^N.
\]
For \(B\in\binom{[N]}\ell\), let \(A_B=[N]\setminus B\). If the unit-margin
constraints on \(A_B\) are feasible, define
\[
 \widehat u_B
 :=
 \argminop_u
 \left\{
   \frac12\norm u^2:
   Y_i\ip{u}{X_i}\ge1\quad(i\in A_B)
 \right\},
\]
\[
 r_B:=\max_{i\in A_B}\norm{X_i},
 \qquad
 \widehat K_B:=r_B^2\norm{\widehat u_B}^2.
\]
Here \(\ell\) is the number of held-out observations, \(N\) is the pooled
sample size, and \(\widehat K_B\) is the normalized squared norm of the
separator trained on the complement of \(B\).

For a deterministic cutoff \(\kappa\ge1\), let
\(I_B^{(\kappa)}(W)\) be the indicator that all of the following hold:
\[
 \begin{split}
 &A_B\text{ is feasible},
 \qquad
 \widehat K_B\le\kappa,\\
 &\norm{X_i}\le r_B,
 \qquad
 Y_i\ip{\widehat u_B}{X_i}\le0
 \quad(i\in B).
 \end{split}
 \tag{2.12}\label{eq:ghost-indicator}
\]
Thus \(I_B^{(\kappa)}(W)=1\) precisely when \(B\) is a held-out set whose
points lie within the retained radius, are misclassified by the separator
trained on the retained sample, and whose normalized squared separator norm
is at most \(\kappa\).

Lowercase \(x_i\) always denotes a deterministic vector in the deletion
theorem. Uppercase \(X_i\) denotes a random input observation. The symbols
\(u_B\) and \(\widehat u_B\) distinguish the abstract deterministic separator
from the separator fitted to a pooled random sample.

\subsection{Main results}

\begin{theorem}[Terminal SVM bound]
\label{thm:main-svm}
For every \(m\ge1\) and every \(\delta\in(0,1)\),
\[
 \boxed{
 \Pp\left(
   \gamma_m>0,\quad
   \Risk(u_m)>
   \frac{C_*}{m}
   \left(
      K_m+\log\frac1\delta
   \right)
 \right)
 \le\delta,}
\]
where one may take
\[
 C_*=(19e^3+1)(1+\log2)<648.
\]
On \(\{\gamma_m>0\}\), Lemma~\ref{lem:margin-norm} gives
\[
 K_m=\frac{r_m^2}{\gamma_m^2}.
\]
Hence the threshold is equivalently
\[
 \frac{C_*}{m}
 \left(
   \frac{r_m^2}{\gamma_m^2}+\log\frac1\delta
 \right).
\]
In particular, the theorem holds with the round numerical constant \(648\).
\end{theorem}

\begin{theorem}[Uniform deletion-family bound]
\label{thm:uniform-count}
Let \(x_1,\ldots,x_n\in\mathbb B_{\Hh}\), let \(k>0\), and let
\(\mathcal F\subseteq\binom{[n]}q\). Suppose that every
\(B\in\mathcal F\) is admissible at squared-norm budget \(k\), in the sense
of (2.2). Then
\[
 \boxed{|\mathcal F|\le e^{8k+2q}.}
\]
\end{theorem}

\begin{corollary}[Additional squared-norm windows do not change the count]
\label{cor:finite-shell}
Under the hypotheses of Theorem~\ref{thm:uniform-count}, suppose additionally
that
\[
 a\le\norm{u_B}^2\le b\le k
 \qquad(B\in\mathcal F)
\]
for arbitrary real numbers \(a\le b\). Then
\[
 |\mathcal F|\le e^{8k+2q}.
\]
In particular, no lower bound on \(\norm{u_B}^2\) and no restriction on the
width \(b-a\) are needed for the upper bound.
\end{corollary}

\begin{proof}
Theorem~\ref{thm:uniform-count} uses only the upper bound
\(\norm{u_B}^2\le k\).
\end{proof}

The deterministic theorem is stronger than the fixed-width squared-norm
window estimate needed in the statistical reduction. Standard background on
convex analysis, probability, and Shannon entropy may be found
in~\cite{R70,B95,CT06}; the specific facts used here are proved in the paper
or collected in Appendix~\ref{sec:convex-appendix}.

\section{From deletion geometry to an entropy bound}
\label{sec:entropy-proof}

Fix the random deletion model of Section~\ref{sec:definitions}. The goal is
to prove
\[
 H(B)\le4\Phi_k(t)+2\E|B|.
\]
The proof recursively reveals deletion bits. At each node of the recursion,
the conditional law is again a deletion law of the same form: coordinates
known to be retained are moved to \(J\), and coordinates known to be deleted
are removed. The only issue is to choose a coordinate whose revealed bit can
be charged to a quantity that decreases down the recursion.

The separator spread \(t\) is the quantity that decreases. The KKT identity
below finds an informative coordinate whenever all remaining deletion
probabilities are at most one half. The exceptional high-probability
coordinates are handled directly by the expected deletion count.
\subsection{An exact accounting identity for separator spread}

The first lemma is the bridge between the geometry and the entropy argument.
The left side is the mean squared spread \(t\) of the random separator. The
right side assigns that spread to the KKT coefficient mass carried by each
coordinate, multiplied by the amount by which the mean separator falls short
of unit score at that coordinate.

\begin{lemma}[KKT--covariance identity and an informative coordinate]
\label{lem:covariance-identity}
Let \(\nu\) be admissible at squared-norm budget \(k\). Define
\[
 \bar\alpha_i:=\E\alpha_i(B),
 \qquad
 m_i:=\ip{\bar u}{x_i}.
\]
Then
\[
 \boxed{
 t=
 \sum_{i\in J\cup V}\bar\alpha_i(1-m_i).}
 \tag{3.1}
\]

Assume in addition that every variable coordinate is genuinely random:
\(0<p_i<1\), where
\[
 p_i:=\Pp(i\in B),
 \qquad
 i\in V.
\]
Set
\[
 a_i:=\E[\ip{U}{x_i}\mid i\notin B],
 \qquad
 b_i:=\E[\ip{U}{x_i}\mid i\in B],
 \qquad
 g_i:=a_i-b_i.
\]
Then
\[
 a_i\ge1,
 \qquad
 b_i\le0,
 \qquad
 g_i\ge1,
\]
and
\[
 \boxed{
 t\le
 \sum_{i\in V}\bar\alpha_i p_i g_i.}
 \tag{3.2}
\]
Consequently, if \(t>0\), some \(i\in V\) satisfies
\[
 \boxed{
 p_i g_i\ge\frac{t}{k}.}
 \tag{3.3}
\]
\end{lemma}

\begin{proof}
By Lemma~\ref{lem:kkt},
\[
 \E\norm U^2
 =
 \E\sum_i\alpha_i(B)
 =
 \sum_i\bar\alpha_i.
\]
Also,
\[
 \bar u
 =
 \E U
 =
 \sum_i\bar\alpha_i x_i,
\]
so
\[
 \norm{\bar u}^2
 =
 \sum_i\bar\alpha_i\ip{\bar u}{x_i}
 =
 \sum_i\bar\alpha_i m_i.
\]
Since
\[
 t=\E\norm U^2-\norm{\bar u}^2,
\]
identity (3.1) follows.

For \(i\in J\), every state retains \(i\), hence \(m_i\ge1\), and its
summand in (3.1) is nonpositive. For \(i\in V\),
\[
 m_i=(1-p_i)a_i+p_i b_i,
\]
and therefore
\[
 1-m_i
 =
 1-a_i+p_i(a_i-b_i)
 \le
 p_i g_i.
\]
This proves (3.2). Finally,
\[
 \sum_i\bar\alpha_i
 =
 \E\norm U^2
 \le k.
\]
If every \(p_i g_i<t/k\), then the right-hand side of (3.2) is strictly less
than \(t\), a contradiction. Hence (3.3) holds for at least one coordinate.
\end{proof}

Identity (3.1) may be read as a conservation law. The total separator spread
\(t\) cannot be positive unless some coordinate combines nonnegligible KKT
weight with a deficit in the mean score. After separating retained and
deleted states, inequality (3.3) says that at least one coordinate has a
large product
\[
 \text{deletion probability}\times\text{conditional score jump}.
\]
The next lemma turns that score jump into an actual decrease of conditional
variance when the deletion bit is revealed.

\begin{lemma}[Variance drop from revealing one deletion bit]
\label{lem:variance-split}
Fix \(i\in V\) with \(0<p_i<1\), write \(p=p_i\), and put
\[
 \bar u_0:=\E[U\mid i\notin B],
 \qquad
 \bar u_1:=\E[U\mid i\in B],
\]
\[
 t_0:=
 \E[\norm{U-\bar u_0}^2\mid i\notin B],
 \qquad
 t_1:=
 \E[\norm{U-\bar u_1}^2\mid i\in B].
\]
Then
\[
 \boxed{
 t=(1-p)t_0+pt_1+\Delta_i,}
 \qquad
 \Delta_i:=
 p(1-p)\norm{\bar u_0-\bar u_1}^2,
 \tag{3.4}
\]
and
\[
 \boxed{
 \Delta_i\ge p(1-p)g_i^2.}
 \tag{3.5}
\]
\end{lemma}

\begin{proof}
Expand \(U-\bar u\) around the appropriate conditional mean and average.
The cross terms vanish conditionally, giving (3.4). Moreover,
\[
 \ip{\bar u_0-\bar u_1}{x_i}
 =
 a_i-b_i
 =
 g_i.
\]
Since \(\norm{x_i}\le1\), Cauchy--Schwarz gives
\[
 \norm{\bar u_0-\bar u_1}\ge g_i,
\]
proving (3.5).
\end{proof}

\subsection{From variance drop to binary-entropy control}

For a split with deletion probability \(p\), the new information is the
binary entropy \(h(p)\). The following scalar inequality shows that, in the
low-deletion-probability regime, the geometric variance drop is large enough
to control that entropy.

\begin{lemma}[Binary entropy controlled by the variance drop]
\label{lem:binary-payment}
Let
\[
 0<p\le\frac12,
 \qquad
 0<A\le1,
 \qquad
 g\ge1,
 \qquad
 pg\ge A,
\]
and suppose
\[
 \Delta\ge p(1-p)g^2.
\]
Then
\[
 \boxed{
 h(p)\le4\Delta\log\frac eA.}
 \tag{3.6}
\]
\end{lemma}

\begin{proof}
For \(0<p<1\),
\[
 -(1-p)\log(1-p)\le p,
\]
because
\[
 -\log(1-p)\le\frac{p}{1-p}.
\]
Hence
\[
 h(p)\le p\log\frac ep.
 \tag{3.7}
\]
Since \(p\le1/2\) and \(g\ge1\),
\[
 \Delta
 \ge
 \frac12pg^2
 \ge
 \frac p2,
 \tag{3.8}
\]
and, using \(pg\ge A\),
\[
 \Delta
 \ge
 \frac{(pg)^2}{2p}
 \ge
 \frac{A^2}{2p}.
 \tag{3.9}
\]
Put
\[
 L:=\log\frac eA\ge1.
\]

If \(p\ge A\), then (3.7)--(3.8) give
\[
 h(p)
 \le
 p\log\frac eA
 \le
 2\Delta L.
\]
If \(p<A\), write \(p=Ar\) with \(0<r<1\). Then
\[
 h(p)
 \le
 Ar\left(L+\log\frac1r\right)
 =
 \frac A r\,r^2\left(L+\log\frac1r\right).
\]
The function \(s\mapsto se^{-2s}\) has maximum \(1/(2e)\) on
\([0,\infty)\). With \(s=\log(1/r)\), this yields
\[
 r^2\log\frac1r\le\frac1{2e}\le L.
\]
Also \(r^2L\le L\), so
\[
 h(p)
 \le
 2\frac A r L
 =
 2\frac{A^2}{p}L
 \le
 4\Delta L
\]
by (3.9). This proves (3.6).
\end{proof}

\subsection{The entropy induction}

There are two kinds of split. If a coordinate is deleted with probability
greater than \(1/2\), then \(h(p)<2p\), so the term \(2\E|B|\) controls the
new binary entropy directly. If all deletion probabilities are at most
\(1/2\), Lemma~\ref{lem:covariance-identity} supplies an informative
coordinate, and the decrease of \(4\Phi_k(t)\) controls the binary entropy.
Conditioning preserves the deletion model because a revealed-retained
coordinate moves into \(J\), whereas a revealed-deleted coordinate is
removed.

\begin{theorem}[Entropy bound for random admissible deletion sets]
\label{thm:entropy}
Let \(\nu\) be any admissible deletion law at squared-norm budget \(k>0\). Then
\[
 \boxed{
 H(B)
 \le
 4\Phi_k(t)+2\E|B|,}
 \qquad
 t=\E\norm{u_B-\E u_B}^2.
 \tag{3.10}
\]
In particular,
\[
 \boxed{
 H(B)\le8k+2\E|B|.}
 \tag{3.11}
\]
\end{theorem}

\begin{proof}
We argue by induction on \(|V|\). The assertion is immediate when \(V\) is
empty.

First remove deterministic coordinates. If \(p_i=0\), then \(i\) is always
retained; move it from \(V\) to \(J\). This changes neither \(B\), \(U\),
\(t\), nor \(\E|B|\). If \(p_i=1\), remove \(i\) from \(V\) and replace
\(B\) by \(B\setminus\{i\}\). Entropy and \(t\) are unchanged, while the
expected cardinality drops by one, so the induction hypothesis for the
reduced law is stronger than the desired conclusion. We may therefore
assume
\[
 0<p_i<1\qquad(i\in V).
 \tag{3.12}
\]

The map \(B\mapsto u_B\) is injective on the support of \(\nu\). Indeed, if
\(B\ne D\), choose \(i\in B\setminus D\), after interchanging the two sets
if necessary. State \(B\) deletes \(i\), whereas state \(D\) retains it, and
hence
\[
 \ip{u_B}{x_i}\le0,
 \qquad
 \ip{u_D}{x_i}\ge1.
\]
Thus \(u_B\ne u_D\). Consequently, if \(\nu\) is not a point mass, then
\(t>0\). The point-mass case has \(H(B)=0\) and is finished, so assume
\(t>0\).

Fix a coordinate \(i\in V\), and reveal the deletion bit
\[
 \xi:=\one_{\{i\in B\}},
 \qquad
 p:=\Pp(\xi=1),
 \qquad
 B':=B\setminus\{i\}.
\]
Conditionally on \(\xi=0\), the coordinate \(i\) becomes permanently retained;
conditionally on \(\xi=1\), it becomes permanently deleted and may be removed
from the ground set. Therefore both conditional laws are admissible at the
same squared-norm budget \(k\) on a variable set of size \(|V|-1\). Let \(t_0,t_1\) be
their separator variances as in Lemma~\ref{lem:variance-split}.

The map \(B\leftrightarrow(\xi,B')\) is one-to-one. Expanding the definition of
Shannon entropy gives the binary chain rule
\[
 H(B)
 =
 h(p)
 +(1-p)H(B'\mid \xi=0)
 +pH(B'\mid \xi=1).
 \tag{3.13}
\]
Also,
\[
 (1-p)\E[|B'|\mid \xi=0]
 +p\E[|B'|\mid \xi=1]
 =
 \E|B|-p.
 \tag{3.14}
\]
Applying the induction hypothesis in the two children and using
(3.13)--(3.14) yields
\[
 \begin{aligned}
 H(B)
 &\le
 h(p)
 +4\left(
      (1-p)\Phi_k(t_0)
      +p\Phi_k(t_1)
   \right)
 +2(\E|B|-p).
 \end{aligned}
 \tag{3.15}
\]
We now choose the splitting coordinate.

Suppose first that some coordinate has \(p>1/2\), and split on any such
coordinate. Since \(\Phi_k\) is concave and increasing, and since
Lemma~\ref{lem:variance-split} gives
\[
 (1-p)t_0+pt_1=t-\Delta_i\le t,
\]
we have
\[
 (1-p)\Phi_k(t_0)+p\Phi_k(t_1)
 \le
 \Phi_k(t-\Delta_i)
 \le
 \Phi_k(t).
 \tag{3.16}
\]
Moreover,
\[
 h(p)\le\log2<1<2p.
 \tag{3.17}
\]
Substitution into (3.15) proves (3.10).

It remains to consider the case
\[
 p_i\le\frac12\qquad(i\in V).
 \tag{3.18}
\]
By Lemma~\ref{lem:covariance-identity}, choose \(i\) so that
\[
 p_i g_i\ge\frac tk.
\]
Set
\[
 A:=\frac tk\in(0,1].
\]
Lemmas~\ref{lem:variance-split} and~\ref{lem:binary-payment} give
\[
 h(p_i)
 \le
 4\Delta_i\log\frac{ek}{t}.
 \tag{3.19}
\]
Concavity of \(\Phi_k\) and (3.4) imply
\[
 (1-p_i)\Phi_k(t_0)+p_i\Phi_k(t_1)
 \le
 \Phi_k(t-\Delta_i).
 \tag{3.20}
\]
Because a differentiable concave function lies below each of its tangent
lines,
\[
 \Phi_k(t)-\Phi_k(t-\Delta_i)
 \ge
 \Delta_i\Phi_k'(t)
 =
 \Delta_i\log\frac{ek}{t}.
 \tag{3.21}
\]
Combining (3.19)--(3.21),
\[
 h(p_i)
 +4\left(
      (1-p_i)\Phi_k(t_0)
      +p_i\Phi_k(t_1)
   \right)
 \le
 4\Phi_k(t).
\]
Equation (3.15) now proves (3.10), completing the induction.

Finally,
\[
 t\le\E\norm U^2\le k
\]
and \(\Phi_k(t)\le2k\), so (3.11) follows.
\end{proof}

The theorem has a direct coding interpretation. The random set \(B\) carries
\(H(B)\) nats of information. The possible movement of the corresponding
minimum-norm separator contributes at most \(8k\) nats, while the expected
number of deleted coordinates contributes at most \(2\E|B|\) nats.

\begin{proof}[Proof of Theorem~\ref{thm:uniform-count}]
If \(\mathcal F=\varnothing\), the conclusion is immediate. Otherwise apply
Theorem~\ref{thm:entropy} with \(J=\varnothing\), \(V=[n]\), and \(B\) uniform
on \(\mathcal F\). Then
\[
 H(B)=\log|\mathcal F|,
 \qquad
 \E|B|=q.
\]
The second bound in Theorem~\ref{thm:entropy} gives
\[
 \log|\mathcal F|\le8k+2q.
\]
\end{proof}

\begin{remark}[Why dependence on \(k+q\) is unavoidable]
Let \(n=K+q\), where \(K,q\) are positive integers, and take
\[
 x_i=e_i\in\mathbb R^n.
\]
For every \(B\in\binom{[n]}q\),
\[
 u_B=\sum_{i\notin B}e_i,
 \qquad
 \norm{u_B}^2=K,
 \qquad
 \ip{u_B}{e_i}=0\quad(i\in B).
\]
Thus all \(\binom{K+q}{q}\) deletion states are admissible. In the balanced
case \(K=q=s\),
\[
 \binom{2s}{s}\ge\frac{2^{2s}}{2s+1},
\]
because the largest of the \(2s+1\) binomial coefficients in
\[
 \sum_{j=0}^{2s}\binom{2s}{j}=2^{2s}
\]
is the central one. Hence an exponent of order \(K+q\) is necessary in
general.
\end{remark}

\section{From deletion counting to the SVM risk bound}
\label{sec:statistical-reduction}

We now derive Theorem~\ref{thm:main-svm}. There are two population-error
terms. The in-radius term \(R_m^{\le}\) is controlled by the deletion-family
theorem through a ghost sample. The radius tail \(\rho_m\) is controlled
directly by an order-statistic argument.

\subsection{Margin normalization and the population radius tail}

\begin{lemma}[Margin--norm identity]
\label{lem:margin-norm}
On \(\{\gamma_m>0\}\),
\[
 \boxed{
 \norm{u_m}=\frac1{\gamma_m},
 \qquad
 K_m=r_m^2\norm{u_m}^2
 =\frac{r_m^2}{\gamma_m^2},
 \qquad
 K_m\ge1.}
 \tag{4.1}
\]
\end{lemma}

\begin{proof}
If \(u\) is feasible, then \(w=u/\norm u\) is a unit vector satisfying
\[
 \min_i\ip{w}{Z_i}\ge\frac1{\norm u}.
\]
In particular,
\[
 \gamma_m\ge\frac1{\norm{u_m}}.
\]
Conversely, for any unit vector \(w\), put
\[
 s:=\min_i\ip{w}{Z_i}.
\]
If \(s>0\), then \(w/s\) is feasible, so minimality of \(u_m\) gives
\[
 \norm{u_m}\le\frac1s,
\]
or \(s\le1/\norm{u_m}\). The same inequality is trivial when \(s\le0\).
Taking the supremum over \(w\) proves
\[
 \gamma_m=\frac1{\norm{u_m}}.
\]

Finally, for every training index,
\[
 1\le\ip{u_m}{Z_i}\le\norm{u_m}\,r_m.
\]
Thus \(r_m\norm{u_m}\ge1\), which is equivalent to \(K_m\ge1\).
\end{proof}

Recall from Section~\ref{sec:definitions} that
\(\Risk(u_m)\le R_m^{\le}+\rho_m\). The next lemma controls the second term.
It uses only that the training radii are independent draws from the
population radius distribution.

\begin{lemma}[Population mass beyond the empirical maximum]
\label{lem:radius-tail}
For every \(t\in[0,1)\),
\[
 \boxed{
 \Pp(\rho_m>t)
 \le
 (1-t)^m
 \le
 e^{-mt}.}
 \tag{4.2}
\]
No continuity or atomlessness assumption is made on the distribution of
\(\norm X\).
\end{lemma}

\begin{proof}
Let
\[
 R:=\norm X,
 \qquad
 R_i:=\norm{X_i},
 \qquad
 F(s):=P(R\le s).
\]
The empirical radius is \(r_m=\max_{i\le m}R_i\), but the quantity
\[
 \rho_m=P(R>r_m)=1-F(r_m)
\]
is a population tail probability. Thus no approximation of \(P\) by the
empirical distribution is being assumed.

The assertion is trivial for \(t=0\), so fix \(0<t<1\) and define the
\((1-t)\)-quantile
\[
 c_t:=\inf\{s\ge0:F(s)\ge1-t\}.
\]
Since \(F(s)\to1\) as \(s\to\infty\), \(c_t<\infty\). Right-continuity gives
\(F(c_t)\ge1-t\), while the definition of the infimum gives
\[
 F(c_t-)=P(R<c_t)\le1-t.
\]
If \(\rho_m>t\), then \(F(r_m)<1-t\), and therefore \(r_m<c_t\). Hence every
training radius is strictly below \(c_t\), so independence gives
\[
 \Pp(\rho_m>t)
 \le
 \Pp(r_m<c_t)
 =
 \Pp(R_1<c_t,\ldots,R_m<c_t)
 =
 F(c_t-)^m
 \le
 (1-t)^m.
\]
The strict inequality \(R_i<c_t\) and the left limit \(F(c_t-)\) are what
make the argument valid in the presence of atoms. Finally,
\(1-t\le e^{-t}\).
\end{proof}

\subsection{The exact factorial ghost identity}

Use the pooled-sample notation of Section~\ref{sec:definitions}. The
designated set
\[
 B_0:=\{m+1,\ldots,m+\ell\}
\]
consists of \(\ell\) independent test observations placed after the first
\(m\) training observations. Conditional on the training sample, the event
that all of them are in-radius errors has probability
\((R_m^{\le})^\ell\). Exchangeability then permits averaging over every
\(\ell\)-subset of the pooled sample.

\begin{lemma}[Exact factorial ghost identity]
\label{lem:ghost}
For every \(\kappa\ge1\) and every integer \(\ell\ge1\),
\[
 \boxed{
 \E\left[
   \one_{\{\gamma_m>0,\ K_m\le \kappa\}}
   (R_m^{\le})^\ell
 \right]
 =
 \binom{m+\ell}{\ell}^{-1}
 \E\sum_{B\in\binom{[m+\ell]}\ell}I_B^{(\kappa)}.}
 \tag{4.3}
\]
\end{lemma}

\begin{proof}
Conditional on \(S_m\), the \(\ell\) designated ghost
observations are independent, and each satisfies
\[
 \norm X\le r_m,
 \qquad
 Y\ip{u_m}{X}\le0
\]
with conditional probability \(R_m^{\le}\). Therefore
\[
 \E[I_{B_0}^{(\kappa)}\mid S_m]
 =
 \one_{\{\gamma_m>0,\ K_m\le \kappa\}}
 (R_m^{\le})^\ell.
\]
Taking expectations gives the left-hand side of (4.3). By exchangeability of
the pooled sample, every \(\ell\)-subset \(B\) has the same expected
indicator as \(B_0\). Therefore the expectation for one designated ghost
set equals the average expected count over all
\(\binom{m+\ell}{\ell}\) possible ghost sets, which proves the identity.
\end{proof}

\subsection{A common normalization turns ghost sets into a deletion family}

The in-radius condition in \(I_B^{(\kappa)}\) has a second role beyond the
risk split. For every counted \(B\), it forces the largest retained radius to
equal the largest radius in the entire pooled sample. Consequently all
candidate deletion sets use the same scaling. After this common
normalization, the sets counted by \(I_B^{(\kappa)}\) form exactly an
admissible deterministic deletion family with \(q=\ell\) and squared-norm
budget \(k=\kappa\).

\begin{proposition}[Uniform factorial-moment bound]
\label{prop:moment}
For every \(\kappa\ge1\) and every integer \(\ell\ge1\),
\[
 \boxed{
 \E\left[
   \one_{\{\gamma_m>0,\ K_m\le \kappa\}}
   (R_m^{\le})^\ell
 \right]
 \le
 \frac{e^{8\kappa+2\ell}}{\binom{m+\ell}{\ell}}.}
 \tag{4.4}
\]
\end{proposition}

\begin{proof}
Fix a realization of the pooled sample and let
\[
 \mathcal F_W
 :=
 \left\{
 B\in\binom{[N]}\ell:
 I_B^{(\kappa)}(W)=1
 \right\}.
\]
If \(\mathcal F_W\) is empty there is nothing to prove. Otherwise put
\[
 R:=\max_{1\le i\le N}\norm{X_i}.
\]
For every \(B\in\mathcal F_W\), definition~\eqref{eq:ghost-indicator} says that all deleted
radii are at most the retained maximum \(r_B\). Hence the full-pool maximum is
exactly \(r_B\), so
\[
 r_B=R>0
 \qquad(B\in\mathcal F_W).
\]
Define normalized signed vectors and normalized separators by
\[
 x_i:=\frac{Y_iX_i}{R},
 \qquad
 v_B:=R\widehat u_B.
\]
Then \(\norm{x_i}\le1\), and \(v_B\) is the exact minimum-norm separator for
the retained normalized constraints. Moreover,
\[
 \norm{v_B}^2
 =
 R^2\norm{\widehat u_B}^2
 =
 \widehat K_B
 \le \kappa,
\]
and, for every \(i\in B\),
\[
 \ip{v_B}{x_i}
 =
 Y_i\ip{\widehat u_B}{X_i}
 \le0.
\]
Theorem~\ref{thm:uniform-count}, applied with \(q=\ell\) and
\(k=\kappa\), therefore gives
\[
 |\mathcal F_W|\le e^{8\kappa+2\ell}.
\]
Insert this pointwise estimate into Lemma~\ref{lem:ghost}.
\end{proof}

\begin{corollary}[One-layer tail bound]
\label{cor:layer}
Let \(\kappa\ge1\), \(\eta\in(0,1)\), and
\[
 \ell:=
 \left\lceil8\kappa+\log\frac1\eta\right\rceil.
\]
Then
\[
 \boxed{
 \Pp\left(
   \gamma_m>0,\ K_m\le \kappa,\quad
   R_m^{\le}>e^3\frac\ell m
 \right)
 \le\eta.}
 \tag{4.5}
\]
\end{corollary}

\begin{proof}
The elementary product formula gives
\[
 \binom{m+\ell}{\ell}
 =
 \prod_{j=1}^{\ell}\frac{m+j}{j}
 \ge
 \left(\frac m\ell\right)^\ell.
\]
By Proposition~\ref{prop:moment},
\[
 \E\left[
   \one_{\{\gamma_m>0,K_m\le \kappa\}}
   (R_m^{\le})^\ell
 \right]
 \le
 e^{8\kappa+2\ell}
 \left(\frac\ell m\right)^\ell.
\]
On the event in (4.5),
\[
 (R_m^{\le})^\ell
 >
 e^{3\ell}\left(\frac\ell m\right)^\ell.
\]
Hence the pointwise inequality
\[
 \one_{\{R_m^{\le}>e^3\ell/m\}}
 \le
 \frac{(R_m^{\le})^\ell}
      {e^{3\ell}(\ell/m)^\ell}
\]
gives
\[
 \Pp\left(
   \gamma_m>0,\ K_m\le \kappa,\quad
   R_m^{\le}>e^3\frac\ell m
 \right)
 \le
 e^{8\kappa-\ell}
 \le
 \eta.
\]
This is Markov's inequality written explicitly for the present random
variable.
\end{proof}

\subsection{Peeling over the random squared-norm statistic}

The moment estimate holds on every event \(\{K_m\le\kappa\}\). We apply it
at integer cutoffs and use a summable allocation of the failure
probability.

\begin{proof}[Proof of Theorem~\ref{thm:main-svm}]
Put
\[
 L_\delta:=\log\frac2\delta.
\]
For \(j\ge1\), define
\[
 \eta_j:=\delta\,2^{-(j+1)},
 \qquad
 \ell_j:=
 \left\lceil8j+\log\frac1{\eta_j}\right\rceil.
\]
Corollary~\ref{cor:layer}, applied with \(\kappa=j\), gives
\[
 \Pp\left(
   \gamma_m>0,\ K_m\le j,\quad
   R_m^{\le}>e^3\frac{\ell_j}{m}
 \right)
 \le
 \eta_j.
 \tag{4.6}
\]
Since
\[
 \sum_{j=1}^{\infty}\eta_j=\frac\delta2,
\]
with probability at least \(1-\delta/2\), none of the events in (4.6)
occurs.

On \(\{\gamma_m>0\}\), Lemma~\ref{lem:margin-norm} gives \(K_m\ge1\). Let
\[
 j:=\lceil K_m\rceil.
\]
Then
\[
 j-1<K_m\le j,
 \qquad
 j\le K_m+1\le2K_m.
\]
Since
\[
 \log\frac1{\eta_j}
 =
 \log\frac1\delta+(j+1)\log2
 =
 L_\delta+j\log2,
\]
we have
\[
 \begin{aligned}
 \ell_j
 &\le
 (8+\log2)j+L_\delta+1\\
 &\le
 2(8+\log2)K_m+L_\delta+1\\
 &\le
 19(K_m+L_\delta).
 \end{aligned}
 \tag{4.7}
\]
Therefore, outside an event of probability at most \(\delta/2\),
\[
 \gamma_m>0
 \quad\Longrightarrow\quad
 R_m^{\le}
 \le
 \frac{19e^3}{m}(K_m+L_\delta).
 \tag{4.8}
\]

Apply Lemma~\ref{lem:radius-tail} with
\[
 t=\frac{L_\delta}{m}
\]
when this number is less than one. It gives
\[
 \Pp\left(
   \rho_m>\frac{L_\delta}{m}
 \right)
 \le
 e^{-L_\delta}
 =
 \frac\delta2.
\]
If \(L_\delta/m\ge1\), the same inequality
\[
 \rho_m\le\frac{L_\delta}{m}
\]
holds deterministically because \(\rho_m\le1\). Consequently, with
probability at least \(1-\delta/2\),
\[
 \rho_m\le\frac{L_\delta}{m}.
 \tag{4.9}
\]

On the intersection of the good events (4.8) and (4.9),
inequality~\eqref{eq:risk-split} yields
\[
 \Risk(u_m)
 \le
 \frac{19e^3+1}{m}
 \left(
   K_m+\log\frac1\delta+\log2
 \right).
\]
Since \(K_m\ge1\),
\[
 K_m+\log\frac1\delta+\log2
 \le
 (1+\log2)
 \left(
   K_m+\log\frac1\delta
 \right).
\]
A union bound over the two exceptional events proves the stated \(K_m\)-form
of the theorem with
\[
 C_*=(19e^3+1)(1+\log2).
\]
The margin form follows on \(\{\gamma_m>0\}\) from
Lemma~\ref{lem:margin-norm}.
\end{proof}

\section{Conclusion}

The proof rests on one exact bridge between convex geometry and information.
For a random admissible deletion set \(B\), the possible minimum-norm
separators form a cloud \(U=u_B\) in the Hilbert space. The KKT coefficients
give the identity
\[
 \E\norm{U-\E U}^2
 =
 \sum_i\E\alpha_i(B)
 \left(
   1-\ip{\E U}{x_i}
 \right).
\]
Because the total coefficient mass is at most \(k\), positive separator
spread forces at least one coordinate whose deletion probability times its
retained-versus-deleted score jump is large.

Revealing that coordinate has a concrete geometric effect: the two
conditional mean separators are separated in the direction \(x_i\), so the
law of total variance removes a definite amount from the remaining
conditional spread. The concave potential
\[
 4t\log\frac{e^2k}{t}
\]
is chosen precisely so that this decrease controls the binary entropy of the
revealed deletion bit. Coordinates deleted more than half the time are
handled directly by \(2\E|B|\). Recursion gives
\[
 H(B)\le8k+2\E|B|,
\]
and the uniform law on a \(q\)-deletion family yields
\[
 |\mathcal F|\le e^{8k+2q}.
\]

The statistical transfer contains one further geometric observation. For a
ghost set counted by \(I_B^{(\kappa)}\), every deleted radius is no larger
than the largest retained radius. Therefore the retained maximum equals the
maximum radius of the full pooled sample, independently of \(B\). This
common normalization turns all counted ghost sets into one deterministic
deletion family. The factorial ghost identity then gives the in-radius error
bound, while Lemma~\ref{lem:radius-tail} separately controls the population
mass beyond the empirical maximum.

\appendix

\section{Convex-analytic background}
\label{sec:convex-appendix}

This appendix contains the standard Hilbert-space projection theorem and the
finite-constraint KKT certificate used in the main argument. Their proofs
are included for completeness but are not part of the covariance--entropy
mechanism. The final lemma records the convex-hull formula and the
measurability facts needed for the statistical statement.

\begin{lemma}[Hilbert projection principle]
\label{lem:hilbert-projection}
Let \(C\) be a nonempty closed convex subset of a real Hilbert space and let
\(x\) be a point of that space. There is a unique \(p\in C\) minimizing
\(\norm{x-p}\). It is characterized by
\[
 \boxed{\ip{x-p}{z-p}\le0\qquad(z\in C).}
 \tag{A.1}
\]
Conversely, any \(p\in C\) satisfying this inequality is the unique nearest
point to \(x\).
\end{lemma}

\begin{proof}
Put \(d=\inf_{z\in C}\norm{x-z}\) and choose \(p_r\in C\) with
\(\norm{x-p_r}\to d\). Convexity and the parallelogram identity give
\[
 \begin{aligned}
 \norm{p_r-p_s}^2
 &=2\norm{x-p_r}^2+2\norm{x-p_s}^2
   -4\norm{x-\tfrac12(p_r+p_s)}^2\\
 &\le2\norm{x-p_r}^2+2\norm{x-p_s}^2-4d^2.
 \end{aligned}
\]
Thus \((p_r)\) is Cauchy. Its limit \(p\) belongs to \(C\) and realizes the
minimum. Uniqueness follows from the same identity, or from strict convexity
of the squared norm.

For \(z\in C\) and \(0<t\le1\), the point \(p+t(z-p)\) lies in \(C\).
Minimality therefore implies
\[
 0\le
 \norm{x-p-t(z-p)}^2-\norm{x-p}^2
 =
 -2t\ip{x-p}{z-p}+t^2\norm{z-p}^2.
\]
Divide by \(t\) and let \(t\downarrow0\) to obtain (A.1). Conversely, if
(A.1) holds, then
\[
 \norm{x-z}^2
 =
 \norm{x-p}^2+\norm{z-p}^2-2\ip{x-p}{z-p}
 \ge\norm{x-p}^2,
\]
with equality only for \(z=p\).
\end{proof}

\begin{lemma}[Minimum-norm separator and KKT certificate]
\label{lem:kkt}
Let \(S\) be a finite index set and let \(x_i\in\Hh\) for \(i\in S\).
Assume that
\[
 \mathcal C_S
 :=
 \{u\in\Hh:\ip{u}{x_i}\ge1\text{ for every }i\in S\}
\]
is nonempty. Then \(\mathcal C_S\) contains a unique vector \(u_S\) of
minimum norm. There are coefficients \(\alpha_i\ge0\) such that
\[
 u_S=\sum_{i\in S}\alpha_i x_i,
 \qquad
 \alpha_i>0
 \quad\Longrightarrow\quad
 \ip{u_S}{x_i}=1.
\]
Consequently,
\[
 \boxed{\sum_{i\in S}\alpha_i=\norm{u_S}^2.}
 \tag{A.2}
\]
For \(S=\varnothing\), the same statement holds with \(u_S=0\) and the empty
certificate.
\end{lemma}

\begin{proof}
Assume \(S\ne\varnothing\). The set \(\mathcal C_S\) is nonempty, closed,
and convex, so Lemma~\ref{lem:hilbert-projection}, applied with \(x=0\),
gives its unique minimum-norm point \(u_S\).

Let
\[
 L=\operatorname{span}\{x_i:i\in S\}.
\]
Orthogonal projection onto \(L\) preserves every constraint value and cannot
increase the norm, hence \(u_S\in L\). Let
\[
 A:=\{i\in S:\ip{u_S}{x_i}=1\}
\]
be the active set, and put
\[
 K:=\operatorname{cone}\{x_i:i\in A\}\subseteq L.
\]
The cone \(K\) is closed. Indeed, every element of a finitely generated cone
has a representation using a linearly independent subfamily: from any
representation with a linearly dependent positive support, one may move the
coefficient vector along a nonzero dependence until at least one coefficient
vanishes, while preserving nonnegativity. Iterating gives an independent
support. Along a convergent sequence in \(K\), pass to a subsequence using
one fixed independent subfamily. The corresponding coefficient vectors then
converge, because the synthesis map on that independent subfamily has a
continuous inverse on its range. The limit therefore remains in \(K\).

Suppose for contradiction that \(u_S\notin K\), and let \(c\) be the metric
projection of \(u_S\) onto \(K\), whose existence and characterization are
given by Lemma~\ref{lem:hilbert-projection}. Put
\[
 h:=c-u_S.
\]
The projection inequality gives
\[
 \ip{u_S-c}{z-c}\le0\qquad(z\in K).
\]
Using \(z=0\) and \(z=2c\) shows
\[
 \ip{u_S-c}{c}=0.
\]
It follows that
\[
 \ip{h}{z}\ge0\quad(z\in K),
 \qquad
 \ip{h}{u_S}=-\norm{u_S-c}^2<0.
\]
For every active constraint, \(\ip{h}{x_i}\ge0\). Every inactive constraint
has positive slack, and there are only finitely many of them. Therefore, for
all sufficiently small \(\varepsilon>0\), the point
\(u_S+\varepsilon h\) is still feasible. At the same time,
\[
 \norm{u_S+\varepsilon h}^2
 =
 \norm{u_S}^2
 +2\varepsilon\ip{u_S}{h}
 +\varepsilon^2\norm h^2
 <
 \norm{u_S}^2
\]
for sufficiently small \(\varepsilon\), a contradiction. Thus \(u_S\in K\),
so
\[
 u_S=\sum_{i\in A}\alpha_i x_i
\]
for suitable \(\alpha_i\ge0\). Extend the coefficients by zero on
\(S\setminus A\). Taking the inner product with \(u_S\) and using activity
on the positive support gives
\[
 \norm{u_S}^2
 =
 \sum_i\alpha_i\ip{u_S}{x_i}
 =
 \sum_i\alpha_i.
\]
\end{proof}

\begin{remark}
The certificate in Lemma~\ref{lem:kkt} need not be unique. Every choice has
the same total mass \(\norm{u_S}^2\), and the main argument may choose one
certificate independently for each deletion state.
\end{remark}

\begin{lemma}[Projection representation and measurability]
\label{lem:projection}
Let \(x_1,\ldots,x_s\in\Hh\) and let
\[
 z_*:=
 \argminop_{z\in\conv\{x_1,\ldots,x_s\}}\norm z.
\]
The unit-margin system
\[
 \ip{u}{x_i}\ge1,\qquad 1\le i\le s,
\]
is feasible if and only if \(z_*\ne0\). In that case its minimum-norm
solution is
\[
 \boxed{u_*=\frac{z_*}{\norm{z_*}^2}.}
 \tag{A.3}
\]
Moreover, \(z_*\) depends continuously on the finite tuple
\((x_1,\ldots,x_s)\), and \(u_*\) is continuous on the open feasible region
\(\{z_*\ne0\}\). If the tuple ranges in a separable Hilbert space, these maps
are measurable with respect to the corresponding finite product Borel
sigma-algebra. In particular, under the separable-support hypothesis of Section~\ref{sec:definitions},
all sample-dependent quantities used in the main text admit Borel versions.
\end{lemma}

\begin{proof}
If the unit-margin system is feasible, then \(0\) cannot lie in the convex
hull: a feasible \(v\) has inner product at least one with every convex
combination of the \(x_i\).

Conversely, assume \(z_*\ne0\). Lemma~\ref{lem:hilbert-projection}, applied
to the compact convex hull, gives
\[
 \ip{z_*}{x_i-z_*}\ge0\qquad(1\le i\le s).
\]
Thus \(z_*/\norm{z_*}^2\) is feasible. If \(v\) is any feasible vector and
\[
 z_*=\sum_i\lambda_i x_i,
 \qquad
 \lambda_i\ge0,
 \qquad
 \sum_i\lambda_i=1,
\]
then
\[
 1\le\ip{v}{z_*}\le\norm v\,\norm{z_*}.
\]
Thus \(\norm v\ge1/\norm{z_*}\), with equality for the displayed candidate,
which proves (A.3).

For continuity, let \(x_i^{(r)}\to x_i\). Choose
\(\lambda^{(r)}\in\Delta_s\) minimizing
\[
 \norm{\sum_i\lambda_i x_i^{(r)}}.
\]
The corresponding objective functions converge uniformly on the compact
simplex
\[
 \Delta_s:=\left\{\lambda\in[0,1]^s:\sum_{i=1}^s\lambda_i=1\right\}.
\] Every subsequence of \((\lambda^{(r)})\) has a further
convergent subsequence, and the associated convex combinations converge to a
minimum-norm point of \(\conv\{x_1,\ldots,x_s\}\). That point is unique, so
every such limit is \(z_*\). Hence the minimum-norm points converge.

Extending \(z/\norm z^2\) by zero at \(z=0\) gives a Borel map on the
topological product \(\Hh^s\). When \(\Hh\) is separable,
\(\mathcal B(\Hh^s)=\mathcal B(\Hh)^{\otimes s}\), so this map is measurable
for the product sigma-algebra used by the sample law. The usual
monotone-class argument for parameterized integrals of nonnegative Borel
functions then shows that the risks and truncated risks defined in Section~\ref{sec:definitions} are
Borel functions of the sample. Under the statistical assumptions of
Section~\ref{sec:definitions}, apply this argument inside \(\Hh_0\) and extend the resulting maps
by fixed values off \(\Hh_0^s\).
\end{proof}

\end{document}